\documentclass[12pt]{article}
\usepackage{amsmath,amssymb,fullpage,graphicx,amscd,array,amsfonts,mathtools,amsthm}
\usepackage{xcolor}
\usepackage{url}
\usepackage[english]{babel}

\usepackage[round,authoryear]{natbib}

\RequirePackage{geometry}
\allowdisplaybreaks

\newtheorem{lemma}{Lemma}[section]

\newtheorem{proposition}{Proposition}[section]
\newtheorem{remark}{Remark}[section]

\newcommand{\pr}{\mathbb{P}}
\newcommand{\E}{\mathbb{E}}

\newcommand{\real}{\mathbb{R}}
\newcommand{\inner}[2]{\langle #1,#2\rangle}
\newcommand{\ind}{1{\hskip -2.5 pt}\hbox{I}}

\let\hat\widehat

\newenvironment{enum}{
\begin{enumerate}
  \setlength{\itemsep}{1pt}
  \setlength{\parskip}{0pt}
  \setlength{\parsep}{0pt}
}{\end{enumerate}}
\begin{document}

%
%
%
%
%

\begin{center}
{\LARGE A Conformal Prediction Approach to Explore\\
\vspace{3mm} Functional Data}\\

\vspace{5mm}

Jing Lei,\ \ Alessandro Rinaldo,\ \ Larry Wasserman \\
\vspace{3mm}
{\it Carnegie Mellon University}\\
\vspace{2mm}
\today
\end{center}

\begin{abstract}
    \centering
  \vspace{3 true pt}
  \begin{minipage}{0.8\textwidth}
\indent
  This paper applies conformal prediction techniques to compute
  simultaneous prediction bands and clustering trees for functional data.
  These tools can be used to detect outliers and clusters.
  Both our prediction bands and clustering trees provide prediction sets
  for the underlying stochastic process with a guaranteed finite sample behavior,
  under no distributional assumptions.
  The prediction sets are also informative in that they
  correspond to the high density region of the underlying process.
  While ordinary conformal prediction has high computational cost for functional data,
  we use the inductive conformal predictor, together with
  several novel choices of conformity scores, to simplify
  the computation.
   Our methods are illustrated on some real data examples.
\end{minipage}
\end{abstract}

\section{Introduction}
\label{intro}
Functional data analysis has been the focus of much research efforts in the statistics
and machine learning community
in the last decade.
In functional data analysis, the data points are functions rather than
scalars or vectors.  The functional perspective provides a powerful modeling
tool for many natural
processes with certain smoothness structures.  Typical examples arise
in temporal-spatial statistics, longitudinal data, genetics, and engineering.
The literature on functional data analysis
is growing very quickly and familiar techniques like
regression, classification, principal component analysis and clustering
have all been extended to
functional data.
The books by
\cite{ramsay} and \cite{ferraty2006nonparametric}
have established the state of the art.

The focus of this paper is exploratory analysis and visualization for
functional
data, including detection of outliers, median sets, and high
density sets.
Due to its infinite dimensional nature, visualization is a challenging
problem for functional data since it is hard to see directly which curves
are typical and which are abnormal.
In the literature, many classical notions and tools
 for ordinary vector data have
been extended to functional data, using
either finite dimensional projection methods such as
functional principal components (FPCA) and spline basis, or
componentwise methods.  For example, \cite{HyndmanS10}
developed functional bagplots and boxplots, where
a band in functional space is obtained by first
applying bivariate bagplots \cite{RousseeuwRT99} on the first
two principal component scores, and projecting the bivariate sets
back onto the functional space.  In \cite{SunG11},
the sample functions are ordered using
a notion of functional data depth \cite{Lopez-PintadoR09} and a
band is constructed by taking componentwise maximum and minimum
of sample quantile functions. Other related work in outlier detection
mostly
use functional data depth to order the curves, and then
apply robust estimators to detect outliers, see
\cite{CuevasFF07}, \cite{FebreroGG-M08}, \cite{Gervini09}.

In this paper, we visualize functional quantiles in the form of
\emph{simultaneous prediction bands}.  That is, for a given level
of coverage, we construct a band that covers a random curve
drawn from the underlying process. Our prediction bands are constructed by
combining
the finite dimensional projection approach and the idea of
\emph{conformal prediction},
a general approach to construct distribution free finite sample valid prediction
sets \cite{VovkGS08b,ShaferV08}.  Although
originally developed as a tool for online
prediction, conformal prediction methods have
been proven to be useful for general
nonparametric inferences,
yielding robust and efficient prediction sets \cite{LeiRW11,LeiRW12}.
To apply conformal methods to functional data, a major challenge
is computation efficiency.
Computing the bands with finite dimensional projection
is infeasible using ordinary conformal
prediction methods because the conformal prediction sets
are usually hard to characterize.
Here we use the inductive conformal method
which allows efficient implementation for functional data when
combined with some carefully chosen conformity scores. The resulting
prediction bands
always give correct finite sample coverage, without
any regularity conditions on the underlying process.
Unlike many existing functional boxplots, our bands tend to reflect
the ``high density'' regions in the functional spaces and may
have disconnected slices.
In some cases, such a property can also reveal
other salient structures in the data such as clusters.

Another contribution of this paper is
the construction of clustering trees with finite sample interpretation
for functional
data. In classical vector cases, clustering,
together with the aforementioned prediction sets,
is closely related to density level sets \cite{Hartigan75,RinaldoW10,RinaldoSNW10}.  However,
in functional spaces, the density is not well defined because
there is no $\sigma$-finite dominating measure.
In \cite{ferraty2006nonparametric}, a notion of ``pseudo density''
is used instead of density.  In context of functional data, most methods
use finite dimensional projections, combined with classical clustering
methods such as
K-means \cite{TarpeyK03,AntoniadisBCP10} and Gaussian mixtures
\cite{JamesS03,ShiW08}.  A componentwise approach to functional data clustering is studied by \cite{DelaigleHB12}.
In this paper we construct
 clustering trees for functional data by combining the pseudo density idea
  \cite{ferraty2006nonparametric} and conformal prediction.
Each level of our cluster tree corresponds to an estimated
prediction set, with distribution free, finite sample coverage.
The clustering trees are therefore naturally indexed by the level of coverage.
Moreover, the prediction sets at a higher
level on the tree correspond to regions with higher pseudo density.

To our knowledge, this paper is the first to do prediction and visualization
of functional data with finite sample guarantees.
In Section \ref{sec:func_band}
we introduce the problem of prediction bands for functional data and
the general idea of conformal prediction method under this context.
In Section \ref{sec:extend}
we develop the extended conformal prediction and apply it to efficiently construct
simultaneous prediction bands for function data with good finite sample property.
In Section \ref{sec:pseduo-density} we develop our method of conformal cluster
tree.  Both methods are illustrated through real data examples.

\section{Notation and Background}\label{sec:func_band}
\subsection{Prediction Sets for Functional Data}
Let
$X_1(\cdot), \cdots, X_n(\cdot)$
denote $n$ random functions drawn
from an unknown distribution $P$ over the set of functions on $[0,1]$ with finite energy: ${\Omega}= L_2[0,1]$.
The distribution $P$ is defined on an appropriate
$\sigma$-field ${\cal F}$ on ${\Omega}$.
Given the data and a number
$0<\alpha < 1$, we will construct a set of functions
${\cal C}_n \subset {\Omega}$ such that, for all $P$ and all $n$,
\begin{equation}\label{eq::first-goal}
\mathbb{P}(X_{n+1}\in {\cal C}_n)\geq 1-\alpha
\end{equation}
where $X_{n+1}$ denotes a future function drawn from $P$
and $\mathbb{P}$ denotes the probability corresponding to $P^{n+1}$, the product measure induced by $P$.

Requiring
(\ref{eq::first-goal})
to hold is very ambitious and may be more than needed.
If we are only interested in the main structural features
of the curve, then
we instead aim for the more modest goal that, for all $P$ and all $n$,
\begin{equation}
\mathbb{P}(\Pi(X_{n+1})\in {\cal C}_n)\geq 1-\alpha
\end{equation}
where $\Pi$ is a mapping into a finite dimensional function space
${\Omega}_p\subset {\Omega}$.
The projection $\Pi$ may correspond to the subspace spanned by
the first few functions in Fourier basis, wavelet basis, or any other
orthonormal basis of $L_2[0,1]$.

In practice it is often of interest (for example, for visualization purposes)
to have prediction bands.
A prediction band $\mathcal B_n$ is a prediction set of the form
$\mathcal B_n=\{X(\cdot)\in L_2[0,1]: X(t)\in B_n(t),~\forall~t\in [0,1]\}$, where for each
$t$, $B_n(t)\subseteq\real^1$ can be
expressed as the union of finitely many intervals.  Most existing
prediction bands for functional data with provable coverage relies on
the assumption of Gaussianity
(see, for example, \cite{YaoMW05}).
It is desirable to construct distribution free, finite sample prediction
bands for functional data under general distributions.

\subsection{Conformal Prediction Methods and Variants}

Conformal inference
\cite{VovkGS08b,ShaferV08} is a very general theory of prediction, focused on
sequential prediction.
For our purposes, we only need the following batch version.
Given the observed objects
(random variables, random vectors, random functions, etc.)
$X_1,\ldots, X_n$
we want a prediction set for a new object $X_{n+1}$.
Assume that $X_i\in {\Omega}$ and
let $x\in {\Omega}$ be a fixed, arbitrary object.
We will test the null hypothesis that
$X_{n+1}=x$ and then take ${\cal C}_n$ to be the set of $x$'s that are
not rejected by the test.
Here are the details.

Define the augmented data
${\sf aug}(x) = \{X_1,\ldots, X_n,x\}$.
Define {\em conformity scores}
$\sigma_1,\ldots, \sigma_{n+1}$
where $\sigma_i = g(X_i,{\sf aug}(x))$ for
some function $g$.
(Actually, in \cite{VovkGS08b} they omit $X_i$ from $g$
when defining $\sigma_i$. We discuss this point at the end of this section.)
For $i=n+1$, the score is defined to be
$\sigma_{n+1} = g(x,{\sf aug}(x))$.
The conformity score measures how similar $X_i$ is to
the rest of the data.
An example is
$\sigma_i = - \int (X_i(t) - \overline{X}^x(t))^2 dt$
where
\begin{center}
$
\overline{X}^x(t) = (n+1)^{-1}[x(t)+\sum_{i=1}^n X_i(t)].
$
\end{center}
In fact, we will use some other conformity scores
that are better adapted to the
data distribution.

Consider testing the null hypothesis
$H_0: X_{n+1}=x$.
When $H_0$ is true,
the objects in the set
${\sf aug}(x) = \{X_1,\ldots, X_n,x\}$
are exchangeable and so
the ranks of the conformity scores are uniformly distributed.
Thus, the p-value
\begin{equation}
\pi(x) = \frac{\sum_{i=1}^{n+1} \ind(\sigma_i \leq \sigma_{n+1})}{n+1}
\end{equation}
is uniformly distributed over $\{1/(n+1),2/(n+1),...,1\}$
and is a valid p-value for the test in the sense that
$\mathbb P[\pi_n(X_{n+1})\ge \alpha]\ge 1-\alpha$.
We now invert the test, that is, we collect all the non-rejected null hypotheses:
\begin{equation}
{\cal C}_n = \{ x:\ \pi(x) \geq \alpha\}.
\end{equation}
It follows from the above argument that
\begin{equation}
\mathbb{P}(X_{n+1}\in {\cal C}_n) \geq 1-\alpha\end{equation}
for any $P$ and $n$.
We refer to the above construction as the standard
conformal method.
Next we discuss the use of inductive conformal method which,
as we shall see, has some computational advantages.
The general idea is first given in
Section 4.1 of \cite{VovkGS08b}.

\begin{center}
\fbox{\parbox{4.5in}{
\begin{center}{\sf Algorithm 1:
Inductive Conformal Predictor}\end{center}
 \textbf{Input:} Data $X_1,...,X_n$, confidence level $1-\alpha$, $n_1< n$.\\
 \textbf{Output:} ${\cal C}_n$.
\begin{enum}
\item Split data randomly into two parts $\mathbf
    X_1=\{X_{1},...X_{n_1}\}$, and
$\mathbf X_2=\{X_{n_1+1},...,X_{n}\}$. Let
$n_2=n-n_1$.
\item Let $g:\Omega\mapsto \mathbb R$ be a function constructed from $X_1,\ldots, X_{n_1}$.
\item Define $\sigma_i = g(X_{n_1 +i})$ for $i=1,\ldots, n_2$.
Let $\sigma_{(1)} \leq \cdots \leq \sigma_{(n_2)}$ denote the ranked values.
\item Let ${\cal C}_n=\{x:\ g(x) \geq \lambda\}$ with
$\lambda = \sigma_{(\lceil (n_2+1)\alpha\rceil -1)}$.
\end{enum}
}}
\end{center}

\begin{lemma}\label{lem:gen_conformal}
Let $\mathcal C_n$ be the output of
Algorithm 1.  For all $n$ and $P$,
$\mathbb P(X_{n+1}\in {\cal C}_n)\geq 1-\alpha$.
\end{lemma}

\begin{proof}
  Note that the
  random function $g(\cdot)$ is independent of
  $\mathbf X_2$.
  Assume $X_{n+1}$ is another random sample.
By exchangeability, the rank of $\sigma_{n_2+1}:=g(X_{n+1})$
   among
  $\{\sigma_{1},...,\sigma_{n_2+1}\}$ is
  uniform among $\{1,2,...,n_2+1\}$.
  Therefore with probability at least $1-\alpha$, $X_{n+1}$ falls
  in $\mathcal C_n$.
\end{proof}

In the rest of this paper, unless otherwise noted, we use $n_1=\lfloor n/2\rfloor$.
The data splitting step might seem inefficient due to the sample splitting.
However, it greatly reduces the computational burden
of the standard conformal method by avoiding re-fitting
the function $g$ with every augmented data ${\sf aug}(x)$
for all $x\in \Omega$.  Moreover, such a reduction
can also substantially improve the robustness of the resulting prediction sets.
For example, consider
$k$-means clustering in $d$-dimensional
Euclidean space.  Let $\mathbf Z=(Z_1,...,Z_n) \subset \real^d$ be a data
set.
Given $k$ centers
$v_1,\ldots, v_k\subset \real^d$,
let
$$
R(v_1,...,v_k) = \frac{1}{n}\sum_{i=1}^n \min_j \|v_j-Z_i\|_2^2.
$$
The $k$-means prototypes are the functions
$\hat v_1,\ldots, \hat v_k$ that minimize $R$.
As usual, we can partition the data into $k$
groups $\hat G_1,\ldots, \hat G_k$
where $\hat G_j = \{ Z_i:\|Z_i-\hat v_j\|\leq \|Z_i- \hat v_r\|, r\neq j\}$.

For the standard conformal method,
let
$\hat v_1(z),\ldots, \hat v_k(z)$
denote the prototypes based on the augmented data
$Z_1,\ldots,Z_n,z$.
Define the conformity score
$\sigma_i = -\min_j \|Z_i-\hat v_j(z)\|$.
Let ${\cal C}_n$ denote the resulting conformal set.
With obvious modification,
let
${\cal C}_n'$ denote the conformal set
based on the modified method.
Define ${\sf diam}(\mathcal{C}) = \sup_{x,y\in \mathcal C} \|x-y\|$.

\begin{proposition}\label{thm:infinite_size}
For any $\mathbf Z$,
${\sf diam}({\cal C}_n)=\infty$ but
${\sf diam}({\cal C}_n')<\infty$ if
$n > 2 / \alpha$.
\end{proposition}

\begin{proof} For the first statement,
consider the augmented data
$Z_1,\ldots,Z_n,z$, where
$\|z\|\ge C$.
For $C$ sufficiently large (may depend on $\mathbf Z$),
there exists a group $G_j$ such that
$G_j = \{z\}$.
In this case, $\sigma_{n+1}=0$ and hence
${\cal C}_n\supseteq\{z:\|z\|\ge C\}$ for all $\alpha$.
It follows that ${\sf diam}({\cal C}_n)=\infty$.
The second statement follows since
all the prototypes are in
${\sf convhull}(\mathbf Z_1)\subseteq
{\sf convhull}(\mathbf Z)$. For all $\alpha$,
there exists a $C>0$ large enough such that
$g(z)<\sigma_{1}$
for all
$\|z\|\ge C$. The key is that in the modified method,
the prototypes and the majority of $\sigma_i$'s are not affected
by absurd values of $z$.
\end{proof}

\begin{remark}
Bounded prediction sets could also be obtained by
omitting $Y_i$ from $g$
when defining $\sigma_i$.
This is, in fact, the method used in
\cite{VovkGS08b}. But this ``omit one'' approach is much more computationally expensive
than our data-splitting method.
\end{remark}

\subsection{Conformal Prediction Bands in the Functional Case}
We conclude this section with a brief discussion on
conformal prediction bands for functional data.
When the $X_i$'s are functions, we can construct prediction
bands as follows.
Given ${\cal C}_n$, a conformal prediction set,
we can define upper and lower bounds for all $t\in [0,1]$:
$\ell(t) = \inf_{x\in {\cal C}_n} x(t)$
and $u(t) = \sup_{x\in {\cal C}_n} x(t).$
It follows that
$$\mathbb{P}[ \ell(t) \leq X_{n+1}(t)
\leq u(t),~\forall~t]\geq 1-\alpha.$$
However, these bounds may be hard to compute, depending on
the choice of conformity score.

The choice of conformity score also affects statistical efficiency.
To see this, consider the following example.
Under mild conditions on the process,
the random variable $Y = \sup_t |X(t)|$
has finite quantiles.
Then we can construct conformal prediction
set $\mathcal C_n$ using the standard approach
with conformity score $g(X)=-\sup_{t}|X(t)|$.
Thus $\mathcal C_n$ is a valid prediction set and naturally leads to a band.
However, such a band is usually too wide and hence of
limited use.

We therefore face two challenges:
choosing a good conformity score and
extracting useful information from ${\cal C}_n$.
In the vector setting,
\cite{LeiRW11} showed that
using a density estimator to define a conformity score
leads to conformal set ${\cal C}_n$ with certain
minimax optimality properties.
However, in the present setting, densities do not even exist.
Instead, we consider two approaches:
the first uses density of coefficients of projections to construct prediction bands
 and the second uses
pseudo-densities to build clustering trees.

There is also a third challenge: identifying an optimal conformity score.
This was done using minimax theory in
\cite{LeiRW11,LeiRW12}
in the case where the data are vectors.
To our knowledge, choosing an optimal conformity score
for the functional case is still an open question.
Rather, our current focus is on computation and visualization.

\section{Prediction Bands Based on Projections}\label{sec:extend}

A common approach to functional data analysis is to project the curves
onto a finite dimensional space that captures the main
features of the curves.
In this section, we consider the projection approach because it
enables us to characterize the prediction sets and to
construct simultaneous prediction bands with finite sample guarantee.
To be concrete,
we require the prediction band ${\cal C}_n$ to satisfy
$\mathbb{P}(\Pi(X_{n+1})\in {\cal C}_n)\geq 1-\alpha$
where $\Pi$ is a mapping into a $p$-dimensional function space
${\Omega}_p\subset {\Omega}$.
There are two types of projections:
projections on a fixed basis (for example, Fourier basis, wavelet basis,
and spline basis) and
projections on a data-driven basis such as functional principal components (FPC).
Our method, summarized in Algorithm 2, is general enough so
that it can be used for any basis.
It is
a specific implementation of the general method given
 in Algorithm 1.

 \begin{center}
 \fbox{\parbox{4.5in}{
 \begin{center}{\sf Algorithm 2:
 Functional Conformal Prediction Bands}\end{center}
  \textbf{Input:} Data $X_1,...,X_n$, basis functions $(\phi_1,...\phi_p)$,
  conformity score $\hat f$, level $\alpha$, $1\le n_1<n$.\\
  \textbf{Output:} $B_n(t)\subseteq\real^1$ for all $t\in[0,1]$.
 \begin{enumerate}
 \item Split data randomly into two parts $\mathbf
     X_1=\{X_{1},...X_{n_1}\}$, and
 $\mathbf X_2=\{X_{n_1+1},...,X_{n}\}$. Let
 $n_2=n-n_1$.
 \item 
 Compute basis projection coefficients
     $\xi_{ij}=\inner{X_i}{\phi_j}$, for $i=1,...,n$,
 $j=1,...,p$.  Denote $\xi_i=(\xi_{i1},...,\xi_{ip})$.
 \item For $i=n_1+1,...,n$, evaluate $\hat f(\xi_{i})
 =\hat f(\xi_i;\xi_1,...,\xi_{n_1})$ and rank these
     numbers
 by $f_1\le f_2\le...\le f_{n_2}$.
 \item Define $T_n=\{\xi\in \real^p: \hat f(\xi)\ge\lambda= f_{\lceil
     (n_2+1)\alpha\rceil-1}\}$.
 \item $B_n(t)=\{\sum_{j=1}^p \zeta_j\phi_j(t): (\zeta_1,...,\zeta_p)\in
     T_n\}$.
 \end{enumerate}
 }}
 \end{center}

\begin{proposition}
  Denote $\Pi_p$ the projection operator induced by
  eigen-functions $\{\phi_j:1\le j\le p\}$.  Then the
  prediction band $\mathcal B_n$ given by Algorithm 2 satisfies
  $$\pr\big\{\left(\Pi_p X_{n+1}\right)(t)\in B_n(t),~\forall~t\in[0,1]
  \big\}\ge 1-\alpha.$$
\end{proposition}
The proof follows from that of Lemma \ref{lem:gen_conformal}.

\begin{remark} The above finite sample coverage guarantee
remains valid if the basis functions are estimated from
$\mathbf{X}_1$, which is the case in most applications
where functional principal
components are used.
\end{remark}

\subsection{Conformal Prediction Bands using
Gaussian Mixture Approximation}
In general one can apply Algorithm 2 with any
basis (possibly a data-driven
one obtained from $\mathbf X_1$).
However, in its abstract form, Algorithm 2 leaves some implementation
issues unsolved.
The most challenging one is the characterization of $T_n$
and $B_n(t)$ which depend on the choice of
conformity score.
Usually $\hat f$ is a density estimator of the
projection coefficient vector $\xi\in\real^p$
since it makes sense to use high density region as a prediction set
for future observations.  For example, \cite{LeiRW11}
use kernel density estimator to construct $\hat f$.
Although kernel density estimators can approximate any
smooth densities,
it is unclear how to keep track
of $\{\xi^T\phi(t): \xi\in T_n\}$, the linear projection of the resulting prediction sets, where
$\phi(t)=(\phi_1(t),...,\phi_p(t))^T$.
However, the inductive conformal method allows us to use
any other conformity score to simplify the computation.
Next we show how to use Gaussian mixture density estimator
to construct such a band.

To motivate the method, consider
the mixture model
$$
X\sim P= \pi_1 P_1+\pi_2 P_2+\cdots+\pi_K P_K,
$$
where each $P_k$ is the distribution of a
Gaussian process on $[0,1]$
and $\pi_k$'s are mixing probabilities satisfying $\sum_k\pi_k=1$
and $\pi_k\ge 0$.
Let $\{\phi_j:j\ge 1\}$ be an orthonormal basis of $L_2[0,1]$, and
$\xi_{j}=\inner{X}{\phi_j}$ be the scores of $X$, where
$\inner{f}{g}=\int fg$.  Then
$(\xi_1,...,\xi_p)$ is distributed as a $p$-dimensional
Gaussian mixture.
For smooth processes, the variance of the projection score on
dimension $j$ in
each component decays quickly as $j$ grows.
Therefore, it is common in functional data analysis the
sequence of scores is truncated
for some small $p$.  The basis that gives fastest decay
corresponds to
the functional principal components,
where $\phi_j$'s are the eigen-functions
of $\Gamma$, the covariance function of $X$:
\begin{center}
$
\Gamma=\E [X\otimes X] - \E X \otimes \E X
=\sum_{j\ge 1}\lambda_j \, [\phi_j\otimes \phi_j],
$
\end{center}
where for functions $f,g\in L_2[0,1]$, we define
$f\otimes g:[0,1]^2\mapsto \mathbb{R}^1$:
$[f\otimes g](s,t)=f(s)g(t).$
We emphasize that our band has correct coverage: (1)
without assuming that the mixture model is correct, and (2)
for all choices of projection dimension and numbers of mixture
components.

Going back to Algorithm 2, we can choose $\hat f$ to be
a Gaussian mixture density estimator with $K$ components and
$\phi_1,...,\phi_p$ the first $p$ eigen-functions of
empirical covariance function obtained from $\mathbf X_1$.
Let $\hat\pi_k,~\hat\mu_k,~\hat\Sigma_k$ be
the estimated mixture proportion, mean, and covariance of the $k$th
component.  Denote $\varphi(\cdot;\mu,\Sigma)$ the density function of
${\rm Norm}(\mu,\Sigma)$.
A rough outer bound of $T_n$, the level set of $\hat f$ at $\lambda$,
 can be obtained by
\begin{equation}
  \label{eq:coarse_approx}
T_n = \{\xi: \hat f(\xi)\ge \lambda\}\subseteq \bigcup_{k=1}^K
\{\xi: \varphi(\xi;\hat \mu_k,\hat\Sigma_k) \ge \lambda/(K\hat\pi_k)\}:=
\bigcup_{k=1}^K T_{n,k}.
\end{equation}
Note that each $T_{n,k}$ is an ellipsoid in $\mathbb{R}^p$ whose
projection
on $\phi(t)$ can be computed easily.
Let $u_k(t)=\sup_{\xi\in T_{n,k}} \xi^T\phi(t)$ and
$\ell_k(t)=\inf_{\xi\in T_{n,k}} \xi^T\phi(t)$.  Both
$u_k$ and $\ell_k$ are available in close form.  Then
we can output $ B_n(t)=\cup_k [\ell_k(t), u_k(t)]$.

\subsection{Two refinements}
The approximation in (\ref{eq:coarse_approx}) is usually conservative
 and improvements are usually possible.  We elaborate this idea in two
 cases.

\paragraph{A better approximation of the density level set.}
First, when the components in the Gaussian mixture are well
separated ($\hat\mu_k$'s
are far from each other), then intuitively
the level set of $\hat f$ at $\lambda$ shall be approximately
the union of level sets of $\varphi(\cdot;\hat\mu_k,\hat\Sigma_k)$
at $\lambda/\pi_k$. Now we make this idea precise and
give an more refined approximation of $T_n$ that also retains
finite sample coverage guarantee.

For all $1\le k,s\le K$, define
$$
\delta_{ks}=\sup_{\xi}\left(\hat\pi_k
\varphi(\xi;\hat\mu_k,\hat\Sigma_k)
\wedge\hat\pi_s\varphi(\xi;\hat\mu_s,\hat\Sigma_s)\right),\ \ \
\delta_k=\sum_{s\neq k} \delta_{ks}.
$$
Roughly speaking, $\delta_{ks}$ measures how much
the two components overlap and $\delta_k$ measures how much
the $k$th component overlaps with all other components.
We note that $\delta_{ks}$ can be computed by
solving a sequence of simple quadratically constrained quadratic
programming (QCQP).
More concretely, for a given value of $c$, we
find $\sup_\xi \varphi(\xi;\hat\mu_k,\hat\Sigma_k)$
under the constraint $\varphi(\xi;\hat\mu_s,\hat\Sigma_s)\ge c$,
which is a simple QCQP. Denote this value
by $\eta(c)$, then it can be shown that
$\delta_{ks}$ is either trivial
($\hat\pi_s\varphi(\mu;\hat\mu_s,\hat\Sigma_s)\le \hat\pi_k\varphi(\mu;\hat\mu_k,\hat\Sigma_k)$)
or it can be given by the unique $c^*$ such that $\hat\pi_k c^*=
\hat\pi_s\eta(c^*)$. In the non-trivial case,
$c^*$ can be found using a simple binary search.

When all $\delta_{k}$ are smaller than $\lambda$, we have a
better approximation of $T_n$:
\begin{equation}\label{eq:well_sep_approx}
T_n = \{\xi: \hat f(\xi)\ge \lambda\}\subseteq \bigcup_{k=1}^K
   \{\xi: \varphi(\xi;\hat \mu_k,\hat\Sigma_k) \ge
   (\lambda-\delta_k)/\pi_k\}:=
   \bigcup_{k=1}^K \tilde T_{n,k}.
\end{equation}

Note that $\tilde T_{n,k}$ is also an ellipsoid. We can similarly compute
$\tilde \ell_k=\inf_{\xi\in\tilde T_{n,k}} \xi^T\phi(t)$,
$\tilde u_k=\sup_{\xi\in\tilde T_{n,k}} \xi^T\phi(t)$,
and the band
$\tilde B_n(t)=\cup_{k}[\tilde \ell_k, \tilde u_k]$.
\begin{proposition}
The prediction band $\tilde B_n(t)$ constructed above
is a union of $K$ bands, and
$$
\pr\left\{\exists\,k:\,\left[\Pi_p(X_{n+1})\right](t)\in
[\tilde\ell_k(t),\tilde u_k(t)],\,\forall\,t\right\}\ge 1-\alpha.
$$
\end{proposition}

We implement this method on a real data example.
The data set consists of 1,000 recordings of neurons over time (see Figure
\ref{fig:data_plots} (a)).
The data
 come from a behavioral experiment, performed at the Andrew Schwartz
 motorlab\footnote{\url{http://motorlab.neurobio.pitt.edu/index.php}}:
 A macaque monkey performs a center-out and out-center
 target reaching task with 26 targets in a virtual 3D environment.
 The curves show voltage of neurons versus times recorded at
 electrodes.  The recorded neural activity consists of all action
 potentials detected above a channel-specific threshold on a
 96-channel Utah array implanted in the primary motor cortex.
 One of the goals is ``spike sorting''
 which means clustering the curves.
 Each cluster is thought to correspond to one neuron
 since each neuron tends to have a characteristic curve (spike).

\begin{figure}
  \includegraphics[scale = 0.45]{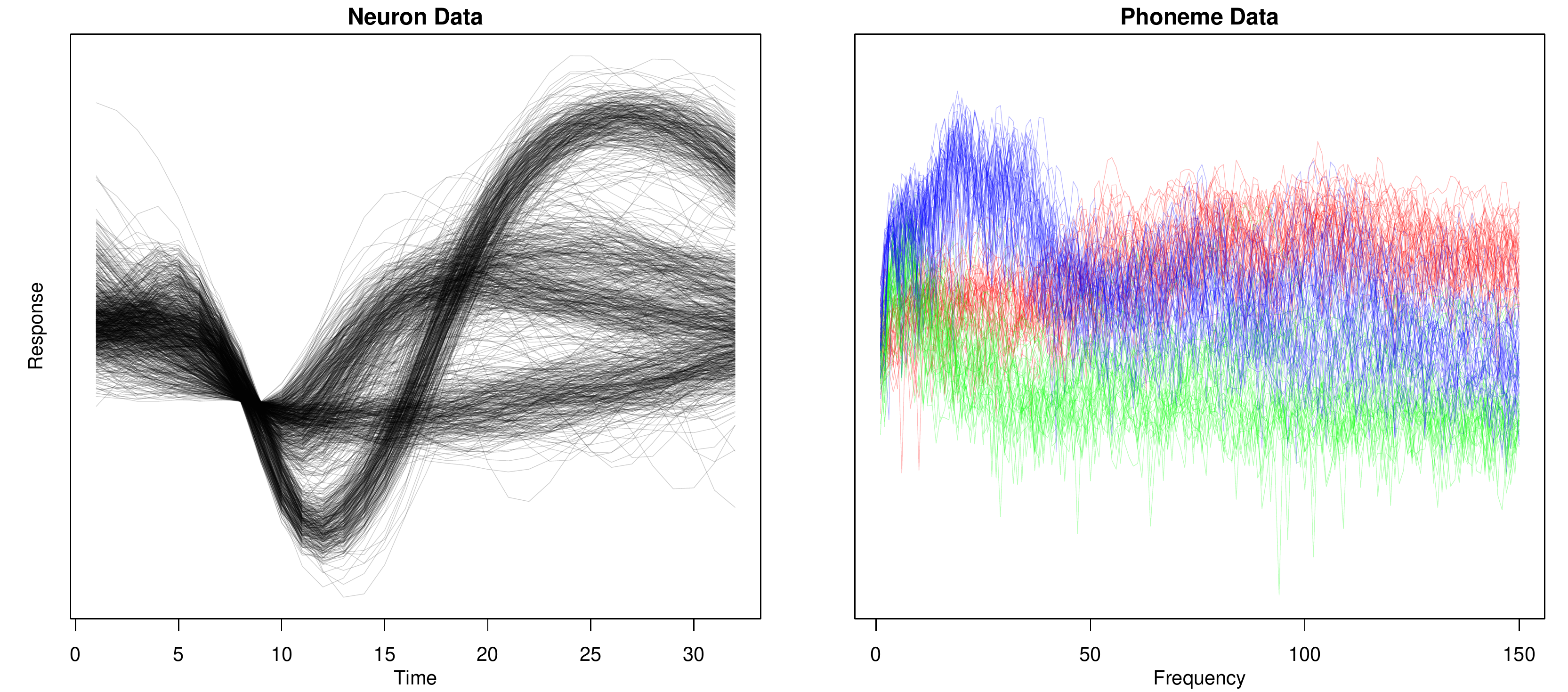}
\caption{Two functional data sets. (a): neuron data.
      (b): phoneme data.}
\label{fig:data_plots}       
\end{figure}

Our implementation uses functional principal components with $p=2$ and $K=3$.
A band consists of three components is plotted along with
the projected sample curves.  The empirical coverage is
916 out of 1000.
See Figure \ref{fig:neuro_bands}.

 \begin{figure}
   \includegraphics[scale = 0.6]{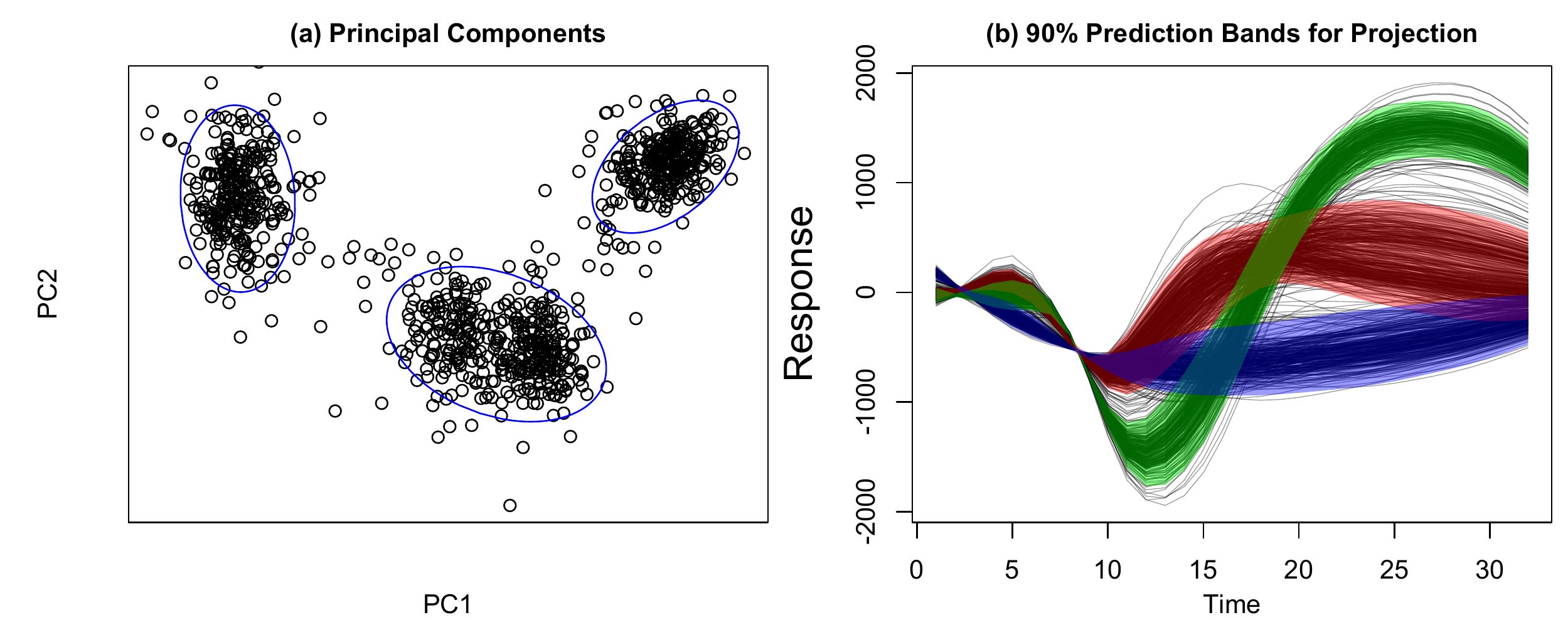}
     \caption{Prediction Bands
     for Neuron Data. (a): plots of first two FPC coefficients
   with boundary of $\tilde T_{n,k}$.
      (b): conformal prediction bands and projected curves.}
     \label{fig:neuro_bands}
 \end{figure}

\paragraph{A different conformity score.} The approximation
introduced above may still be too conservative when the mixture components
are not well separated.  An example of such a situation can be seen
from the phoneme
 data, which is
 considered in \cite{HTF} and \cite{ferraty2006nonparametric}.
 The data considered here
 consists of three phonemes, where each phoneme has
 400 sample curves of discretized periodograms
 of length 150.\footnote{Further information and the data set can be found
 at (\url{http://www.math.univ-toulouse.fr/staph/npfda/npfda-phoneme-des.pdf}).} The data is plotted
 in Figure \ref{fig:data_plots} (b).

When analyzing the  phoneme data, we
also perform a FPCA and focus on the first two principal component scores.
Note that our method can be implemented using other variants such as robust
FPCA and using more dimensions.  Here we choose two principal components
for the ease of visualization.  In the original data, the label of each curve
is known and we summarize this information in
Figure \ref{fig:phone_bands} (a) (but our algorithm assumes that
the labels are unknown).  We observe that the cluster corresponding
to phoneme ``dcl'' exhibits
 significant non-Gaussianity, therefore
 the Gaussian mixture model fitting suggests that this single cluster is
 best modeled by a two-component Gaussian mixture.
 The band for that cluster
 is just the union of the bands given by the two components.

In this case, if the approximation in (\ref{eq:well_sep_approx}) is used,
 the resulting prediction and hence the band will be too wide
 than necessary due to the heavy overlap between the two
 components corresponding to phoneme ``dcl''.
 Here we use a different conformity score
 to obtain a better prediction set. In particular, consider
 \begin{equation}
  \label{eq:new_csf}
  \hat f(\xi)=\max_{k:1\le k\le K}\hat\pi_k
  \varphi(\xi;\hat\mu_k,\hat\Sigma_k).
\end{equation}
That is, instead of computing the sum of density from each component,
we only look at the density of the cluster that $\xi$ is most likely
to belong to.
Such a conformity score function resembles that of a
K-means clustering, except it takes into account of the
mixture probabilities.

The advantage of using such a max function is the decomposability:
\begin{equation}
  \label{eq:decomp_new_csf}
  T_n=\{\xi:\hat f(\xi)\ge \lambda\}=
  \bigcup_{k=1}^K \{\xi:\varphi(\xi;\hat\mu_{k},\hat{\Sigma_k})\ge
  \lambda /\hat\pi_k\}:=\bigcup_{k=1}^K\bar T_{n,k}.
\end{equation}
Since $\bar T_{n,k}$ is also an ellipsoid, we can analogously  define $\bar \ell_k(t)
=\inf_{\xi\in \bar T_{n,k}}\xi\phi(t)$, $\bar u_k(t)=\sup_{\xi\in \bar T_{n,k}}\xi^T\phi(t)$, and
$\bar B_n(t)=\cup_k[\bar \ell_k(t),\bar u_k(t)]$.

\begin{remark} We note that $\bar T_{n,k}$ defined in (\ref{eq:decomp_new_csf}) and $\tilde T_{n,k}$ defined in
(\ref{eq:well_sep_approx}) are very close if the components are
well-separated because $\delta_k\approx 0$ for all $k$ and $\hat f$ defined in (\ref{eq:new_csf})
is very close to the estimated density.\end{remark}

Since the union representation in the above equation is exact, there is no
loss of statistical efficiency.  The only possible loss of efficiency
is using a conformity score function other than the density function.  However,
such a loss is conceivably small: since if the max contributor is large,
the density is likely to be large.  Therefore, ranking the max component
density is roughly the same as ranking the density.
The prediction band for the phoneme data
 is plotted in Figure \ref{fig:phone_bands} (b) where the empirical coverage is
 90.5\%.

 \begin{figure}
   \includegraphics[scale = 0.39]{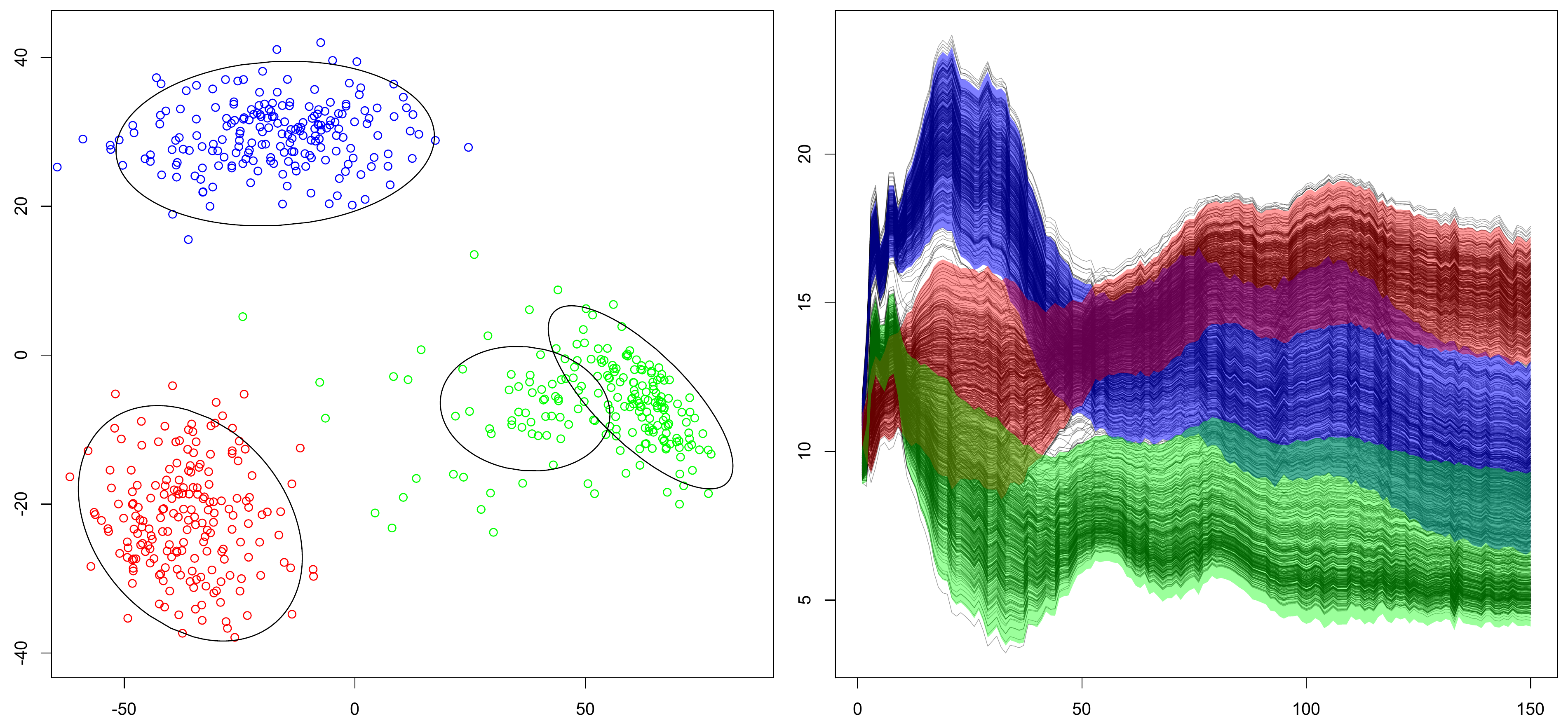}
     \caption{Prediction Bands
     for Phoneme Data. (a): plots of first two FPC coefficients
   with boundary of $\tilde T_{n,k}$.
      (b): conformal prediction bands and projected curves.}
     \label{fig:phone_bands}
 \end{figure}

\section{Methods Based on Pseudo-Densities}\label{sec:pseduo-density}

Here we investigate a different approach, based on pseudo-densities which were introduced by
\cite{ferraty2006nonparametric}. For simplicity, assume $n\alpha$ is an integer.
Given a kernel $K$ and bandwidth $h>0$
define the pseudo-density estimator
\begin{equation}
\hat p_h(u) = \frac{1}{n}\sum_{i=1}^n
K\left( \frac{ d(u,X_i)}{h}\right)
\end{equation}
where
$d(f,g)$ is some distance measure (for example, $d(f,g)=[\int (f-g)^2]^{1/2}$).
We assume that $K(z)\leq K(0)$ for all $z$.
This looks like a kernel density estimator
but it is not a density function.
Indeed, there are no density functions
on ${\Omega}$ because there is no $\sigma$-finite
dominating measure.
Nevertheless,
Ferraty and Vieu show that
$\hat p_h(u)$ can be used for various tasks
such as clustering curves, just as we would do
for a density estimator.
We can view $\hat p_h$ as an estimator of
$p_h(u) = \mathbb{E}(\hat p_h(u)) =
\int K\big( \frac{d(u,x)}{h}\big) dP(x).$

We follow a similar approach here,
and endow the high density sets in the functional space with a conformal
prediction interpretation.
For presentation simplicity, we use the standard conformal method.
Let
${\cal C}_{n,\alpha} = \{ f:\ \pi(f) \geq \alpha \}$
where
\begin{align*}
\pi(f) =&
\frac{1+ \sum_{i=1}^{n} \ind( \hat p_h^f (X_i) \leq \hat p_h^f (f) ) }{n+1},\\
\hat p_h^f(u) =&
\frac{n}{n+1}\hat p_h(u) + \frac{1}{n+1}
K\left(\frac{d(u,f)}{h}\right).
\end{align*}

It is straightforward to see that
$\mathcal C_{n,\alpha}$ constructed above is a level
$1-\alpha$ conformal prediction set, and hence has
distribution-free, finite sample coverage.
The set ${\cal C}_{n,\alpha}$ can be approximated by
a level set of $\hat p_h$ as follows.
Let
$X_{(1)}, X_{(2)},\ldots, $
denote the re-ordered data,
ranked so that
$\hat p_h(X_{(1)}) \leq \hat p_h(X_{(2)}) \leq \cdots$.
Let $\lambda = \hat p_h( X_{(n\alpha)})$ and let
\begin{equation}\label{eq:C+}
{\cal C}_{n,\alpha}^+ =
\left\{f:\ \hat p_h(f) \ge
 \lambda - n^{-1}K(0)\right\}.
\end{equation}
We have the following result
ensuring that the level sets of
the pseudo-density
have a well-defined finite sample predictive interpretation.
Its proof is analogous to that of ordinary
kernel density for vectors (see \cite{LeiRW11}).
\begin{lemma}
We have  $\mathbb{P}( X_{n+1}\in {\cal C}_{n,\alpha}) \geq 1-\alpha$
for all $P$ and $n$.  Furthermore,
${\cal C}_{n,\alpha} \subseteq {\cal C}_{n,\alpha}^+$ and hence,
$\mathbb{P}( X_{n+1}\in {\cal C}_{n,\alpha}^+) \geq 1-\alpha$
for all $P$ and $n$.
\end{lemma}


\begin{figure}
\includegraphics[scale=.47]{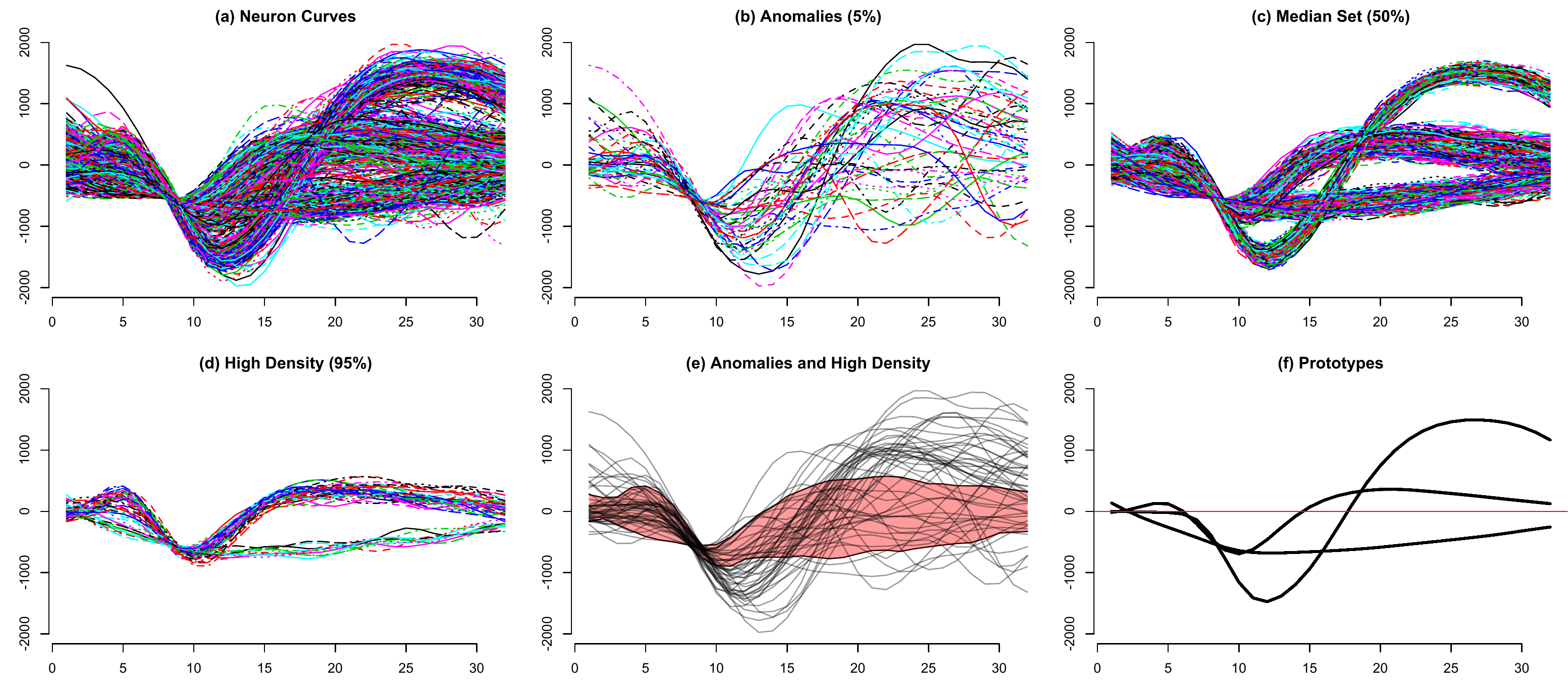} 
  \caption{Neuron data. $X$ axis: time. $Y$ axis: response. (a) 1,000 recordings of action potentials over time.
  (b-e) anormalies, median, and high density sample curves using pseudo-density.
  (f) mean-shift prototypes.}
  \label{fig:group_plot_neuron}
\end{figure}

Now we describe how to make use of and explore the set ${\cal
C}_{n,\alpha}^+$.
Let
$\hat {\mathcal C}_{n,\alpha} = {\cal C}_{n,\alpha}^+ \cap \{X_1,\ldots, X_n\}$.

\paragraph{Anomalies, median set, and high density curves}
We first consider three such sets of $\hat{\cal C}_{n,\alpha}$:
anomalies ($\alpha$ close to 0),
typical ($\alpha = .5$) and
high density ($\alpha$ close to 1).
Figures \ref{fig:group_plot_neuron} (b-e)
show the results using the distance
$d^2(f,g) = \int (f(t)-g(t))^2 dt$ for the neuron data.
It is easily seen that the median set successfully
captures the common shape of each group of neurons.  The
high density set contains curves in the two larger groups.
The anomalies consist of mostly curves with irregular shape.

\paragraph{Modes, prototypes, and mean shift}
Another summary of $\hat{\cal C}_{n,\alpha}$
are prototypes, that is, representative functions corresponding
to the local maxima of the pseudo density (see also the cluster tree below).
The modes of the pseudo-density can be obtained
using the mean-shift algorithm
\cite{cheng1995mean}.
Figure \ref{fig:group_plot_neuron}
(f) shows the three modes obtained in our example.
Two of the modes have the signature of a neuron firing
(a decrease followed by a sharp increase).
The third mode, which stays negative,
is unusual and deserves further attention.

\paragraph{Conformal cluster tree}
Here we use a cluster tree to visualize
how the conformal prediction set evolves as $\alpha$ changes
smoothly from $0$ to $1$.
For a given $\epsilon>0$, define the graph $G_{\alpha,\epsilon}$ whose nodes
correspond to the functions in
$\hat {\cal C}_{n,\alpha}$ and with edges
between nodes
$X_i$ and $X_j$
if $\int (X_i(t) - X_j(t))^2 dt \leq \epsilon^2$.
Define the level $\alpha$ clusters to be the partition of $\hat {\cal
C}_{n,\alpha}$ induced by
connected components of $G_{\alpha,\epsilon}$.
As the conformal parameter $\alpha$ varies from $0$ to $1$, the collection
$\mathcal{T}$ of all level $\alpha$ clusters form a tree (i.e. $A,B \in
\mathcal{T}$ implies that $A \cap B = \emptyset$ or $A \subset B$ or $B
\subset A$), which we call the {\it conformal tree}. The height of the tree is
indexed by $\alpha$, with the root of tree indexed by $\alpha = 0$,
corresponding to all the points in the dataset. As $\alpha$ increases the sets
$\hat {\cal C}_{n,\alpha}$ becomes smaller, and the leaves, which are
associated to local modes of the pseudo density, consist each of a single data
point. Such a conformal tree is similar to an ordinary clustering tree, plus
the additional feature of finite sample coverage and a different but quite natural
indexing on $\alpha$ rather than the usual indexing on the density itself.

The conformal tree provides a graphical representation of some distributional
properties of the data and, in particular, of the ``high-density" regions. As
such it allows us to identify salient features about the data that may be
otherwise hard to detect. Figure \ref{fig:trees}(a)
shows the conformity tree of
the neuron data based on the pseudo density estimator described above using
the $L_2$ distance and the Gaussian kernel. It is easy to see how the data
appear to arise from a mixture of three components (compare with Figure
\ref{fig:neuro_bands}), with splits occurring at values of $\alpha$ equal to
approximately $0.08$ and $0.24$. A further, more subtle,  split occurs at
$\alpha = 0.79$, and distinguishes between two groups of curves exhibiting
very similar but ultimately distinct behavior. The leaves  of the tree
correspond to prototypical curves of highest pseudo density within each
component.  In Figure \ref{fig:trees}(b) we present the
conformal tree with $h=\epsilon=1000$ as in Figure \ref{fig:group_plot_neuron}. The tree
is different but suggests strongly the three-component structure.  It also
rises the issue of choosing tuning parameters, an important problem for
future study.

\begin{figure}
    \includegraphics[scale = 0.45]{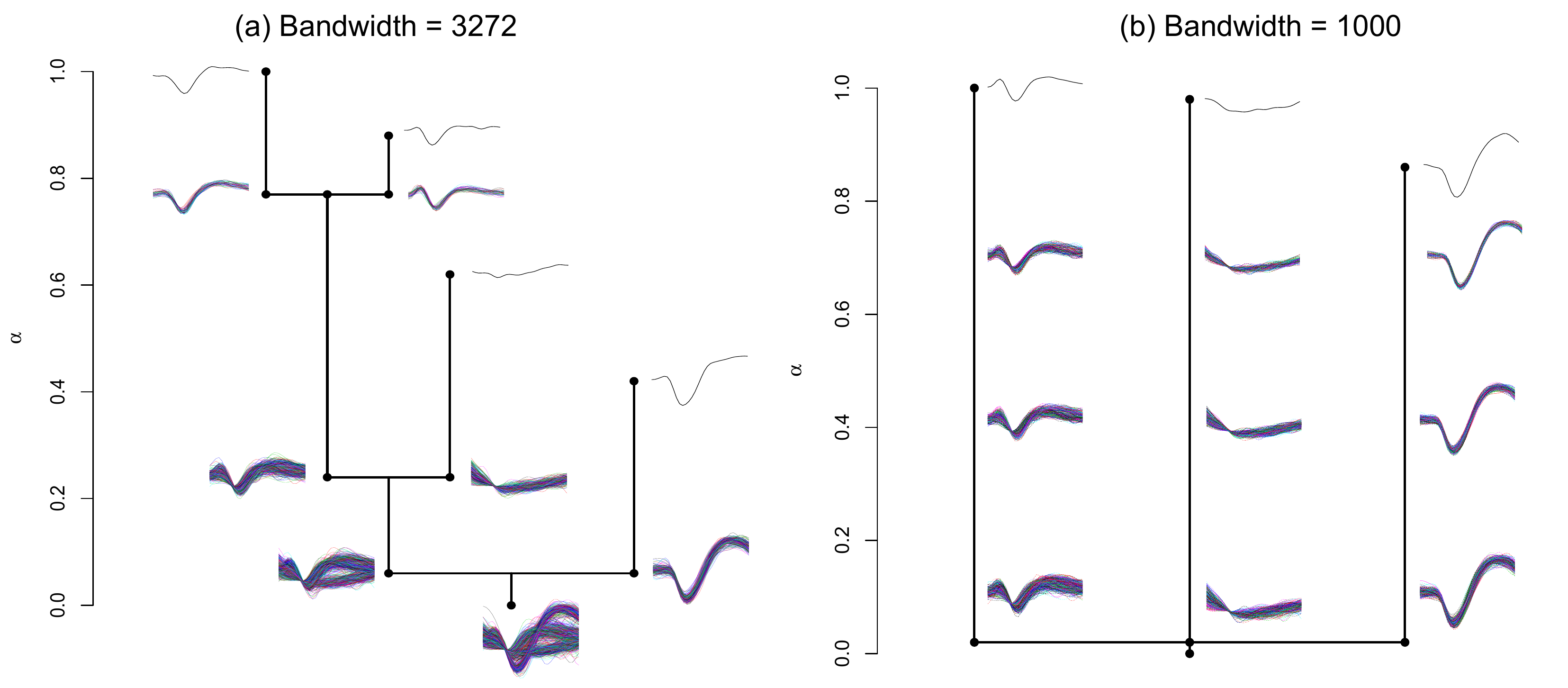}
  \caption{Conformal trees for the neuron data.
  (a): conformal tree with bandwidth $h$ chosen to maximize the variance
  of $\hat p_h(X_1),...,\hat p_h(X_n)$ and $\epsilon$ set to
  $0.5\times\max\{\min_{X_i,X_j\in \hat{\cal C}_{n,\alpha}}d(X_i,X_j)\}$.
  (b): conformal tree obtained using: $h = \epsilon = 1000$. The plots show
  the clusters for different values of $\alpha$.
Next to each node (e.g., split), we plot the curves belonging
to the elements of the corresponding partition of $\hat {\cal C}_{n,\alpha}$.
Isolated clusters of size smaller than 10 are ignored.}
  \label{fig:trees}
\end{figure}

\section{Conclusion and Future Directions}
We have extended conformal methods and its variants to functional data.
Using the inductive conformal predictor, together with basis projections,
we obtain the first
distribution-free simultaneous prediction bands
for functional data with finite sample performance guarantee.
These prediction bands can also be viewed as
quantile sets of the underlying process that corresponding to
high density region.
On the other hand,
using standard conformal prediction with pseudo densities,
the prediction set ${\cal C}_n$ also reveals hierarchical and
salient features in the data.
We have proposed methods
for extracting information from and visualizing ${\cal C}_n$.
Currently, we are investigating several open questions such
as extension of the new conformal method to other conformity scores
in high dimensional problems and
issues regarding optimality and the selection of tuning parameters.
For functional data, the choice of distance in pseudo densities
is also interesting.
It can be shown that for the $L_2$ distance,
the induced bands are unbounded.
The results using the distance
$d^2(f,g) = \sum_j e^{\gamma j}(\beta_j - \theta_j)^2$
for example are somewhat different (not shown).
Here,
$\{\beta_j\}$ and $\{\theta_j\}$
are the coefficients $f$ and $g$ in the cosine basis.
We call this the analytic distance as it
is related to the set of analytic functions; see
page 46 of \cite{efromovich1999nonparametric}.
This distance is much more sensitive to the shape of the curve.

\section*{Acknowledgments}
The authors thank Andy Schwartz, Valerie Ventura and Sonia Todorova for providing
 the data.  Jing Lei is
partially supported by National Science Foundation Grant BCS-0941518. Alessandro Rinaldo is
partially supported by National Science Foundation CAREER Grant DMS-1149677. Larry Wasserman is
partially supported by National Science Foundation Grant DMS-0806009 and Air Force Grant FA95500910373.



\bibliographystyle{apa-good}      
\bibliography{Curves}   

\begin{thebibliography}{25}
\expandafter\ifx\csname natexlab\endcsname\relax\def\natexlab#1{#1}\fi
\expandafter\ifx\csname url\endcsname\relax
  \def\url#1{{\tt #1}}\fi
\expandafter\ifx\csname urlprefix\endcsname\relax\def\urlprefix{URL }\fi

\bibitem[{Antoniadis et~al.(2010)Antoniadis, Xavier~Brossat, \&
  Poggi}]{AntoniadisBCP10}
Antoniadis, A., Xavier~Brossat, J.~C., \& Poggi, J. (2010).
\newblock Clustering functional data using wavelets.
\newblock In {\em e-book of COMPSTAT\/}, (pp. 697--704). Springer.

\bibitem[{Cheng(1995)}]{cheng1995mean}
Cheng, Y. (1995).
\newblock Mean shift, mode seeking, and clustering.
\newblock {\em Pattern Analysis and Machine Intelligence, IEEE Transactions
  on\/}, {\em 17\/}(8), 790--799.

\bibitem[{Cuevas et~al.(2007)Cuevas, Febrero, \& Fraiman}]{CuevasFF07}
Cuevas, A., Febrero, M., \& Fraiman, R. (2007).
\newblock Robust estimation and classification for functional data via
  projection-based depth notions.
\newblock {\em Computational Statistics\/}, {\em 22\/}, 481--496.

\bibitem[{Delaigle et~al.(2012)Delaigle, Hall, \& Bathia}]{DelaigleHB12}
Delaigle, A., Hall, P., \& Bathia, N. (2012).
\newblock Componentwise classification and clustering of functional data.
\newblock {\em Biometrika\/}, {\em 99\/}.

\bibitem[{Efromovich(1999)}]{efromovich1999nonparametric}
Efromovich, S. (1999).
\newblock {\em Nonparametric curve estimation: methods, theory and
  applications\/}.
\newblock Springer Verlag.

\bibitem[{Febrero et~al.(2008)Febrero, Galeano, \&
  {Gonz\'{a}lez-Manteiga}}]{FebreroGG-M08}
Febrero, M., Galeano, P., \& {Gonz\'{a}lez-Manteiga}, W. (2008).
\newblock Outlier detection in functional data by depth measures, with
  application to identify abnormal nox levels.
\newblock {\em Environmetrics\/}, {\em 19\/}.

\bibitem[{Ferraty \& Vieu(2006)}]{ferraty2006nonparametric}
Ferraty, F., \& Vieu, P. (2006).
\newblock {\em Nonparametric functional data analysis: theory and practice\/}.
\newblock Springer Verlag.

\bibitem[{Gervini(2009)}]{Gervini09}
Gervini, D. (2009).
\newblock Detecting and handling outlying trajectories in irregularly sampled
  functional datasets.
\newblock {\em The Annals of Applied Statistics\/}, {\em 3\/}, 1758--1775.

\bibitem[{Hartigan(1975)}]{Hartigan75}
Hartigan, J. (1975).
\newblock {\em Clustering Algorithms\/}.
\newblock John Wiley, New York.

\bibitem[{Hastie et~al.(2009)Hastie, Tibshirani, \& Friedman}]{HTF}
Hastie, T., Tibshirani, R., \& Friedman, J. (2009).
\newblock {\em The Elements of Statistical Learning\/}.
\newblock Springer, 2nd ed.

\bibitem[{Hyndman \& Shang(2010)}]{HyndmanS10}
Hyndman, R.~J., \& Shang, H.~L. (2010).
\newblock Rainbow plots, bagplots, and boxplots for functional data.
\newblock {\em Journal of Computational and Graphical Statistics\/}, {\em
  19\/}, 29--45.

\bibitem[{James \& Sugar(2003)}]{JamesS03}
James, G.~M., \& Sugar, C.~A. (2003).
\newblock Clustering for sparsely sampled functional data.
\newblock {\em Journal of the American Statistical Association\/}, {\em 98\/},
  397--408.

\bibitem[{Lei et~al.(2013)Lei, Robins, \& Wasserman}]{LeiRW11}
Lei, J., Robins, J., \& Wasserman, L. (2013).
\newblock Efficient nonparametric conformal prediction regions.
\newblock {\em Journal of the American Statistical Association\/}.
\newblock To appear.

\bibitem[{Lei \& Wasserman(2013)}]{LeiRW12}
Lei, J., \& Wasserman, L. (2013).
\newblock Distribution free prediction bands for nonparametric regression.
\newblock {\em Journal of the Royal Statistical Society: Series B (Statistical
  Methodology)\/}.
\newblock To appear.

\bibitem[{{L\'{o}pez-Pintado} \& Romo(2009)}]{Lopez-PintadoR09}
{L\'{o}pez-Pintado}, S., \& Romo, J. (2009).
\newblock On the concept of depth for functional data.
\newblock {\em Journal of the American Statistical Association\/}, {\em 104\/},
  718--734.

\bibitem[{Ramsay \& Silverman(1997)}]{ramsay}
Ramsay, J., \& Silverman, B. (1997).
\newblock {\em Functional Data Analysis\/}.
\newblock New York: Springer.

\bibitem[{Rinaldo et~al.(2012)Rinaldo, Singh, Nugent, \&
  Wasserman}]{RinaldoSNW10}
Rinaldo, A., Singh, A., Nugent, R., \& Wasserman, L. (2012).
\newblock Stability of density-based clustering.
\newblock {\em Journal of Machine Learning Research\/}, {\em 13\/}, 905--948.

\bibitem[{Rinaldo \& Wasserman(2010)}]{RinaldoW10}
Rinaldo, A., \& Wasserman, L. (2010).
\newblock Generalized density clustering.
\newblock {\em The Annals of Statistics\/}, {\em 38\/}, 2678--2722.

\bibitem[{Rousseeuw et~al.(1999)Rousseeuw, Ruts, \& Tukey}]{RousseeuwRT99}
Rousseeuw, P.~J., Ruts, I., \& Tukey, J.~W. (1999).
\newblock The bagplot: a bivariate boxplot.
\newblock {\em The American Statistician\/}, {\em 53\/}, 382--387.

\bibitem[{Shafer \& Vovk(2008)}]{ShaferV08}
Shafer, G., \& Vovk, V. (2008).
\newblock A tutorial on conformal prediction.
\newblock {\em Journal of Machine Learning Research\/}, {\em 9\/}, 371--421.

\bibitem[{Shi \& Wang(2008)}]{ShiW08}
Shi, J., \& Wang, B. (2008).
\newblock Curve prediction and clustering with mixtures of gaussian process
  functional regression models.
\newblock {\em Statistics and Computing\/}, {\em 18\/}, 267--283.

\bibitem[{Sun \& Genton(2011)}]{SunG11}
Sun, Y., \& Genton, M.~G. (2011).
\newblock Functional boxplots.
\newblock {\em Journal of Computational and Graphical Statistics\/}, {\em
  20\/}, 216--334.

\bibitem[{Tarpey \& Kinateder(2003)}]{TarpeyK03}
Tarpey, T., \& Kinateder, K. K.~J. (2003).
\newblock Clustering functional data.
\newblock {\em Journal of Classification\/}, {\em 20\/}, 93--114.

\bibitem[{Vovk et~al.(2005)Vovk, Gammerman, \& Shafer}]{VovkGS08b}
Vovk, V., Gammerman, A., \& Shafer, G. (2005).
\newblock {\em Algorithmic Learning in a Random World\/}.
\newblock Springer.

\bibitem[{Yao et~al.(2005)Yao, M\"{u}ller, \& Wang}]{YaoMW05}
Yao, F., M\"{u}ller, H.-G., \& Wang, J.-L. (2005).
\newblock Functional data analysis for sparse longitudinal data.
\newblock {\em Journal of the American Statistical Association\/}, {\em
  100\/}(470), 577--590.

\end{thebibliography}

%
%

\end{document}